\documentclass[letterpaper]{article} 
\usepackage{aaai25}  
\usepackage{times}  
\usepackage{helvet}  
\usepackage{courier}  
\usepackage[hyphens]{url}  
\usepackage{graphicx} 
\usepackage{natbib}  
\usepackage{caption} 
\nocopyright
\usepackage{algorithm}
\usepackage{algorithmic}
\usepackage{amsmath}
\usepackage{amsthm}
\usepackage{amssymb}
\newtheorem{theorem}{Theorem}[section]
\usepackage{subcaption}
\usepackage{amsfonts}
\usepackage{hyperref}
\usepackage{newfloat}
\usepackage{listings}
\usepackage{booktabs}
\usepackage{multirow}
\DeclareCaptionStyle{ruled}{labelfont=normalfont,labelsep=colon,strut=off} 
\floatstyle{ruled}
\newfloat{listing}{tb}{lst}{}
\floatname{listing}{Listing}
\title{ASFT: Aligned Supervised Fine-Tuning through Absolute Likelihood}
\author {
    Ruoyu Wang\textsuperscript{\rm 1},
    Jiachen Sun\textsuperscript{\rm 1},
    Shaowei Hua\textsuperscript{\rm 2},
    Quan Fang\textsuperscript{\rm 1}
}
\affiliations {
    \textsuperscript{\rm 1}Beijing University of Posts and Telecommunications\\
    \textsuperscript{\rm 2}Fuzhou University\\
    wruoyu1125@outlook.com, sunjc@bupt.edu.cn, shaoweihua03@gmail.com , qfang@bupt.edu.cn
}

\usepackage{bibentry}

\begin{document}

\maketitle

\begin{abstract}
Direct Preference Optimization (DPO) is a method for enhancing model performance by directly optimizing for the preferences or rankings of outcomes, instead of traditional loss functions. This approach has proven effective in aligning Large Language Models (LLMs) with human preferences.
Despite its widespread use across various tasks, DPO has been criticized for its sensitivity to the effectiveness of Supervised Fine-Tuning (SFT) and its limitations in enabling models to learn human-preferred responses, leading to less satisfactory performance. To address these limitations, we propose Aligned Supervised Fine-Tuning (ASFT), an effective approach that better aligns LLMs with pair-wise datasets by optimizing absolute likelihood for each response, rather than using the Bradley-Terry model, and eliminates the need for a reference model. Through theoretical gradient analysis, we demonstrate that ASFT mitigates the issue where the DPO loss function decreases the probability of generating human-dispreferred data at a faster rate than it increases the probability of producing preferred data. Additionally, we compare ASFT to DPO and its latest variants, such as the single-step approach ORPO, using the latest instruction-tuned model Llama3, which has been fine-tuned on UltraFeedback and HH-RLHF. We evaluated performance on instruction-following benchmarks like MT-Bench and traditional text generation metrics such as BLEU-4 and ROUGE-L.
Extensive experiments demonstrate that ASFT is an effective alignment approach, consistently outperforming existing methods.
\end{abstract}

%

\section{Introduction}

Large language models (LLMs) trained on large datasets have shown exceptional abilities in generating text that resembles human writing \cite{brown2020languagemodelsfewshotlearners,jiang2023mistral7b,zhang2022optopenpretrainedtransformer}. However, their outputs can frequently deviate from human values and intentions, especially in sensitive or nuanced situations. To address these challenges, researchers have created techniques that utilize human feedback to inform the training process, thereby enhancing the alignment of LLMs with human expectations. Reinforcement learning from human feedback (RLHF) \cite{christiano2023deepreinforcementlearninghuman,stiennon2022learningsummarizehumanfeedback,ouyang2022traininglanguagemodelsfollow} is a widely used approach for fine-tuning language models, aiming to enhance their alignment with human values and preferences. While classical RLHF approach \cite{schulman2017proximalpolicyoptimizationalgorithms} has demonstrated notable success in aligning LLMs with human values, it also poses significant challenges. These challenges include the high computational costs involved in training both the reward and policy models, the risk of reward hacking, and the instability of the training process stemming from the complex interactions between the two models.

To overcome the limitations of RLHF, recent research has introduced Direct Preference Optimization (DPO) \cite{rafailov2024directpreferenceoptimizationlanguage}, a method that directly optimizes human preferences. DPO simplifies the RLHF pipeline by eliminating the need for a separate reward model. Instead, it directly uses human preferences to update the language model's parameters. Despite its considerable success across various tasks, DPO still presents certain limitations that can result in less satisfactory performance, as highlighted by previous research \cite{ethayarajh2024ktomodelalignmentprospect,xu2024contrastivepreferenceoptimizationpushing,feng2024analyzingunderstandinglimitationsdpo}. Specifically, DPO has been criticized for its sensitivity to the effectiveness of the Supervised Fine-Tuning (SFT) phase and its requirement for a reference model, which makes it challenging to implement \cite{xu2024contrastivepreferenceoptimizationpushing}.In other words, LLMs that \textbf{lack proper and effective SFT often demonstrate suboptimal performance when undergoing DPO}. Empirical evidence suggests that SFT and instruction tuning are essential for LLMs to comprehend and follow human directives, which is a prerequisite for effectively aligning with curated human feedback \cite{bai2022traininghelpfulharmlessassistant}. Moreover, the DPO loss function tends to \textbf{decrease the probability of generating human-dispreferred data more rapidly than it increases the likelihood of producing preferred data}, thereby hindering the model's capacity to learn and generate human-preferred responses \cite{feng2024analyzingunderstandinglimitationsdpo}. In fact, alignment methods based on the Bradley-Terry model \cite{19ff28b9-64f9-3656-ba40-08326a05748e} all encounter this issue.

In this work, we theoretically analyze the challenges in gradient optimization of preference alignment methods based on the Bradley-Terry model. To solve the problem, we propose Aligned Supervised Fine-Tuning (ASFT), an effective approach that better aligns LLMs with pairwise datasets by \textbf{optimizing the absolute likelihood} for each response, rather than relying on the Bradley-Terry model. To simplify the training process, policy model can learn human preferences solely from SFT without requiring an additional reference model. Our algorithm core is to \textbf{maximize the maximum likelihood estimation of chosen answers and minimize the maximum likelihood estimation of rejected answers}. This approach significantly bolsters the robustness of large models during training in Section \ref{Sec:Optimization}. In summary, ASFT has the following advantages:
\begin{itemize}
    \item \textbf{Lower computational overhead and more efficient training process}: ASFT can align LLMs with human preferences solely during the SFT stage, eliminating the need for a reference model.
    \item \textbf{More reasonable gradient descent strategy}: ASFT can mitigate the issues caused by gradient optimization strategies that impede the model's ability to learn human-preferred responses. 
    \item \textbf{Significant performance on benchmark}: ASFT demonstrated superior performance on MT-Bench and particularly on Arena-Hard, where it outperformed SFT by 48\% and significantly exceeded the results of DPO and its variants.
\end{itemize}

\section{Related Works}
\subsection{Direct Preference Optimization (DPO)}

DPO \cite{rafailov2024directpreferenceoptimizationlanguage} is a method that focuses on aligning LLMs with pair-wise preference data $\left(x, y_w, y_l\right)$ by directly optimizing the relationship between model outputs and human preferences. DPO simplifies the RLHF pipeline by eliminating the need for a separate reward model, instead, it employs a closed-form expression with an optimal policy to reparameterize the reward function $r$:

\begin{equation}
\begin{aligned}
r(x, y)=\alpha \log \frac{\pi_\theta(y \mid x)}{\pi_{\mathrm{ref}}(y \mid x)}+\alpha \log Z(x)
\end{aligned}
\end{equation}

By integrating this reward construction into the Bradley-Terry (BT) ranking objective \cite{19ff28b9-64f9-3656-ba40-08326a05748e}, 

\begin{equation}
\begin{aligned}
p\left(y_w \succ y_l \mid x\right)=\sigma\left(r\left(x, y_w\right)-r\left(x, y_l\right)\right)
\end{aligned}
\end{equation} 

DPO can represent the probability of preference data using the policy model instead of the reward model, resulting in the following objective:

\begin{equation}
\begin{aligned}
& \mathcal{L}_{\mathrm{DPO}}\left(\pi_\theta ; \pi_{\mathrm{ref}}\right)=-\mathbb{E}_{\left(x, y_w, y_l\right) \sim \mathcal{D}}\\
& \left[\log \sigma\left(\alpha \log \frac{\pi_\theta\left(y_w \mid x\right)}{\pi_{\mathrm{ref}}\left(y_w \mid x\right)}-\alpha \log \frac{\pi_\theta\left(y_l \mid x\right)}{\pi_{\mathrm{ref}}\left(y_l \mid x\right)}\right)\right]
\end{aligned}
\end{equation}

This objective allows DPO to optimize the policy directly with respect to the preferences, without the need for an intermediate reward model.

\subsection{Alignment without Reference Model}

Empirical findings suggest that, non-reinforcement learning (non-RL) alignment methods \cite{rafailov2024directpreferenceoptimizationlanguage,azar2023generaltheoreticalparadigmunderstand} are sensitive to the effectiveness of reference model \cite{feng2024analyzingunderstandinglimitationsdpo,tunstall2023zephyrdirectdistillationlm}. In contrast, there have been approaches that aim to build human-aligned language models by conducting SFT using filtered datasets alone \cite{zhou2023limaalignment,haggerty2024selfsupervisedlearningskincancer} indicated SFT with a limited amount of filtered data can effectively enhance the model's capabilities. Furthermore, the proposal and significant performance of certain single-step fine-tuning methods \cite{hong2024orpomonolithicpreferenceoptimization,meng2024simposimplepreferenceoptimization} clearly demonstrate the feasibility of aligning large models without the need for reference model.

ORPO \cite{hong2024orpomonolithicpreferenceoptimization} explores the role of SFT, providing a theoretical foundation for incorporating preference alignment into SFT. SimPo \cite{meng2024simposimplepreferenceoptimization} demonstrates the potential of single-step fine-tuning for aligning large models through extensive experiments. However, the feasibility for aligning large models without reference model is still insufficiently explored from the perspective of gradient strategies.

\subsection{Limitations of Bradley-Terry model in aligning LLMs}

The Bradley-Terry (BT) model, often used in the context of learning from human preferences, has certain limitations when applied to large model alignment, particularly within the framework of reinforcement learning and non-reinforcement learning (non-RL) alignment method. \citeauthor{azar2023generaltheoreticalparadigmunderstand} conducted extensive research on the assumptions underlying the application of the BT model in aligning LLMs. \citeauthor{feng2024analyzingunderstandinglimitationsdpo} concluded that DPO impairs the ability of large models to learn human preference responses, based on their analysis of the optimization process of DPO. In fact, not only DPO but also other alignment methods based on the BT model face this issue. We conducted a detailed discussion on the optimization process of alignment methods based on the BT model and compared them with the optimization strategy of ASFT in Section \ref{Sec:Optimization}.

\section{Method}

In this section, we first review the derivation of the BT model. We then explain why $-\log \sigma(x)$ is commonly used as a loss function in deep learning models. Finally, we derive the ASFT loss function by introducing new scoring rules and calculating the scores for the selected response and the rejected answer separately.

\subsection{Bradley-Terry model derivation}

The Bradley-Terry model is a classic human preference model. It is a probabilistic framework used to predict the outcomes of two competitors (such as individuals or teams). Commonly applied to pairwise comparison data, it estimates and compares the relative abilities of individuals or projects.

The BT model is a probabilistic model, where the probability of the preference data $i>j$ is given by: 

\begin{equation}
P(i>j)=\frac{p_i}{p_i+p_j}
\end{equation}

Where $p_i$ is the positive real fraction of $i$, we can reparameter it in the following form:

\begin{equation}
P(i>j)=\frac{e^{\beta_i}}{e^{\beta_i}+e^{\beta_j}}=\frac{1}{1+e^{-\left(\beta_i-\beta_j\right)}}
\end{equation}

The score function corresponding to the participants is given by: $p_i=e^{\beta_i}$, which is consistent with the sigmoid function.

\begin{equation}
\sigma(x)=\frac{1}{1+e^{-x}}
\end{equation}

Finally, we can use Maximum Likelihood Estimation (MLE) to calculate the score for each player:

\begin{equation}
\arg \min _\beta \sum_{i,j}-\log \sigma\left(\beta_i-\beta_j\right)
\end{equation}

\subsection{-Log sigmoid function}

The -logsigmoid function can be expressed as:
\begin{equation}
-\log (\sigma(x))=-\log \left(\frac{1}{1+e^{-x}}\right)=-\log \sigma(x)
\end{equation}

Through the formula, we can observe that the -logsigmoid function has two fundamental properties:

\begin{itemize}
    \item As $x \rightarrow-\infty, f(x) \approx x$
    \item As $x \rightarrow+\infty, f(x) \rightarrow 0$
\end{itemize}

These properties make the -logsigmoid function an effective choice for a loss function in machine learning models, where it provides both numerical stability and penalizes incorrect predictions effectively.

For the BT model, the larger the difference between score the $r_\phi\left(x, y_w\right)$ and $r_\phi\left(x, y_l\right)$, the smaller the value of the loss function -logsigmoid:

\begin{equation}\label{eq:9}
\mathcal{L}_R\left(r_\phi, \mathcal{D}\right)=-\mathbb{E}_{\left(x, y_w, y_l\right) \sim \mathcal{D}}\left[\log \sigma\left(\beta r_\phi\left(x, y_w\right)-\beta r_\phi\left(x, y_l\right)\right)\right]
\end{equation}

\subsection{Aligned Supervised Fine-Tuning}

Unlike the BT model, which constructs the loss function by comparing two different answers, our approach uses the probability of the response output by the LLMs as its score. Specifically, we design the loss function to maximize the score of the chosen response while minimizing the score of the discarded response.

The score for the chosen response is as follows:

\begin{equation}
\begin{aligned}
S\left(x, y_\omega\right) & =\pi_\theta\left(y_\omega \mid x\right) \\
& =\frac{\pi_\theta\left(y_\omega \mid x\right)}{1-\pi_\theta\left(y_\omega \mid x\right)+\pi_\theta\left(y_\omega \mid x\right)} \\
& =\sigma\left(\log \frac{\pi_\theta\left(y_\omega \mid x\right)}{1-\pi_\theta\left(y_\omega \mid x\right)}\right)
\end{aligned}
\end{equation}

Let $\log \frac{\pi_\theta\left(y_\omega \mid x\right)}{1-\pi_\theta\left(y_\omega \mid x\right)}$ denote $f_\theta\left(x, y_w\right)$, then we define ASFT optimisation objective as:

\begin{equation}
\left\{
\begin{array}{l}
\displaystyle \max _{\pi_\theta} \mathbb{E}_{\left(x, y_\omega, y_l\right) \sim D} 
\; \sigma\left(f_\theta\left(x, y_\omega\right)\right) \\
\displaystyle \min _{\pi_\theta}  \mathbb{E}_{\left(x, y_\omega, y_l\right) \sim D} 
\; \sigma\left(f_\theta\left(x, y_l\right)\right)
\end{array}
\right.
\end{equation}

The above equation is equivalent to

\begin{equation}
\left\{
\begin{array}{l}
\displaystyle \max _{\pi_\theta} \mathbb{E}_{\left(x, y_\omega, y_l\right) \sim D} 
\; \sigma\left(f_\theta\left(x, y_\omega\right)\right) \\
\displaystyle \max _{\pi_\theta} \mathbb{E}_{\left(x, y_\omega, y_l\right) \sim D} 
\; 1 - \sigma\left(f_\theta\left(x, y_l\right)\right)
\end{array}
\right.
\end{equation}

The loss function form of $\mathcal{L}_{align}$ obtained from above is as follows:

\begin{equation}\label{eq:13}
\mathcal{L}_{align}=-\log \sigma\left(f_\theta\left(x, y_w\right)\right)-\log \sigma\left(-f_\theta\left(x, y_l\right)\right)
\end{equation}

The final objective function ASFT in Equation \ref{eq:14} consists of two components: Supervised Fine-Tuning (SFT) loss $\mathcal{L}_{SFT}$ and loss for alignment $\mathcal{L}_{align}$.

\begin{equation}\label{eq:14}
\mathcal{L}_{ASFT}=\mathbb{E}_{\left(x, y_w, y_l\right)}\left[\mathcal{L}_{SFT}+\beta \mathcal{L}_{align}\right]
\end{equation}

$\mathcal{L}_{SFT}$ in Equation \ref{eq:14} follows the standard approach of minimizing the negative log-likelihood (NLL) loss, which is derived from causal language modeling. This objective is aimed at maximizing the probability of generating the target tokens given the input context.

In contrast to alignment methods based on the BT model, which utilize relative likelihoods for aligning large models, $\mathcal{L}_{align}$ in Equation \ref{eq:13} employs absolute likelihoods to further align large models based on paired data. This approach enhances the model's ability to adjust its outputs more precisely according to specific performance criteria, thereby improving its effectiveness in tasks that require fine-grained distinctions between similar inputs.

\section{Improvements in The Optimization Process}\label{Sec:Optimization}

Previous studies have shown that DPO impedes the ability of LLMs to learn human-preferred responses due to limitations in its gradient optimization process. \cite{feng2024analyzingunderstandinglimitationsdpo} Indeed, we find that alignment methods based on BT models suffer from this issue. In this section, we conducted a detailed analysis of the gradients in alignment methods based on the BT model and provided insights from a gradient perspective on SFT plays a crucial role in enhancing the performance of reference-dependent approaches. Subsequently, we discuss the advantages of the gradient strategy employed by ASFT and demonstrate its exceptional performance, even without the need for a reference model derived through SFT.

\subsection{Optimization Process of BT model}

\begin{figure*}[t]
\centering
\subfloat[The gradient vector field and optimization plane of the loss function based on the BT model]{
  \includegraphics[width=0.75\textwidth]{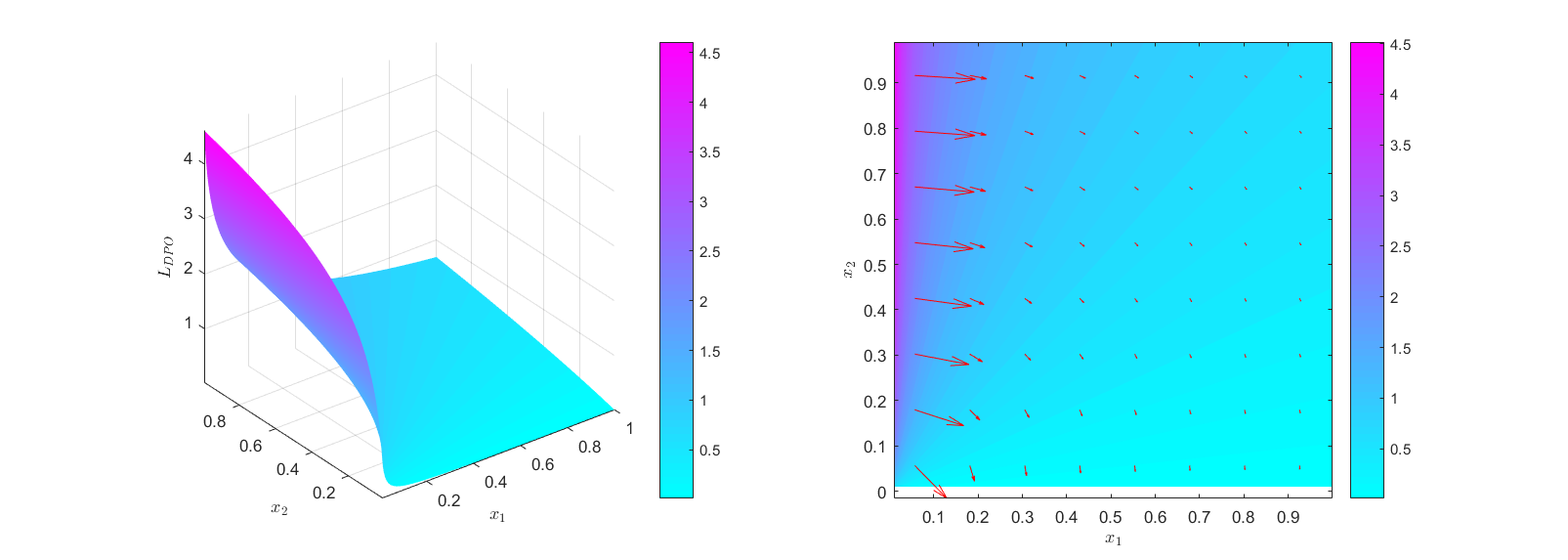}
  \label{fig:a}
}\\ 
\subfloat[The gradient vector and optimization plane field of $\mathcal{L}_{align}$]{
  \includegraphics[width=0.75\textwidth]{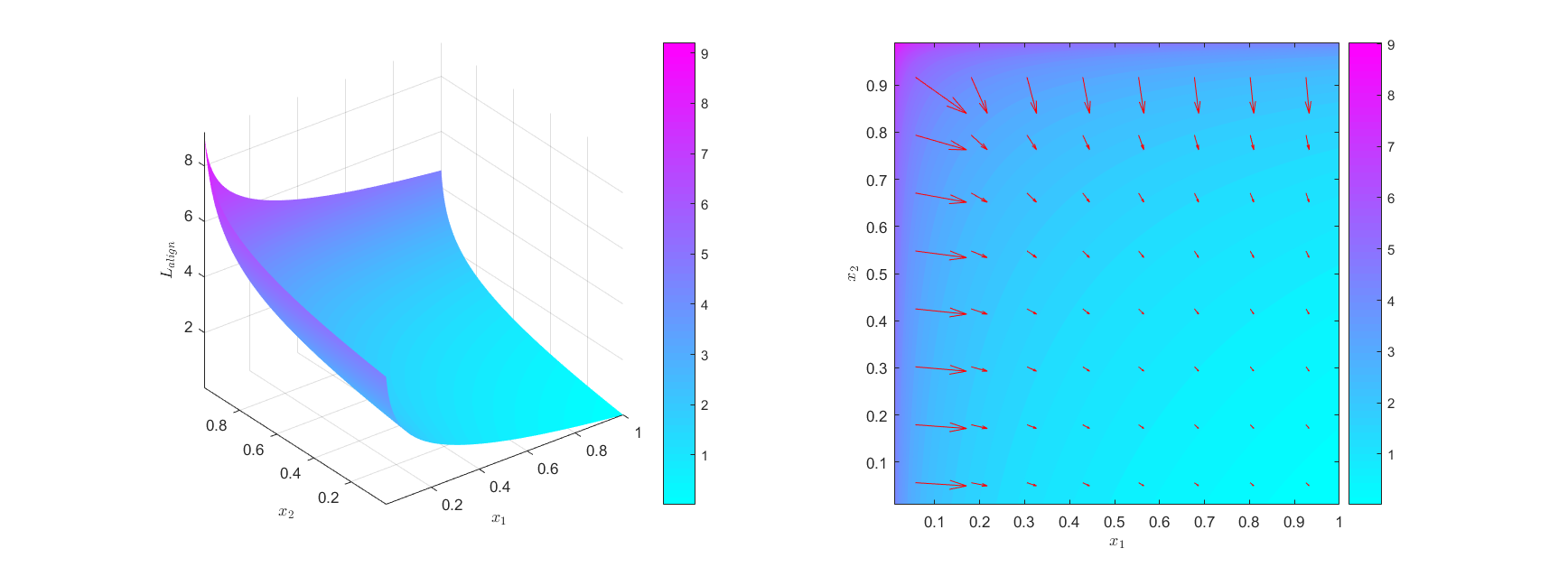}
  \label{fig:b}
}%

\caption{Figures (a) and Figure (b) on the left depict the values of the loss function when preferences and non-preferences responses are generated at different probabilities. On the right, the figures provide a top-down view of the optimization plane of the loss function, showing the gradient field. This is marked with red arrows at various positional points, indicating the direction and magnitude of optimization. The direction of these red arrows represents the path of gradient-based optimization, while the length of the arrows indicates the strength of the optimization.}
\label{fig:2}
\end{figure*}

Existing alignment methods often frame the task as a binary classification problem, employing the Bradley-Terry (BT) model to derive a negative log-likelihood loss in Equation \ref{eq:9}. To simplify the analysis, we denote $r_\phi\left(x, y_w\right)=\log x_1$ and $r_\phi\left(x, y_w\right)=\log x_2$ \footnote{To ensure that $x_1$ reflects the probability of generating a human-preferred response, we adopt a form similar to $r(x, y) \propto \log \pi_\theta(y \mid x)$ in DPO, setting $r_\phi\left(x, y_w\right)=\log x_1$ for consistency. The definition of $x_2$ follows the same principle.} . Thus, the objective function of BT model can be written in the following form:

\begin{equation}\label{eq:15}
\begin{aligned}
\mathcal{L}_R\left(r_\phi, \mathcal{D}\right)= & -\log \sigma\left(\beta r_\phi\left(x, y_w\right)-\beta r_\phi\left(x, y_l\right)\right)\\
= & -\log \left(\frac{x_1^\beta}{x_1^\beta+x_2^\beta}\right)
\end{aligned}
\end{equation}

We denote the transformed form as $\mathcal{L}_{R}\left(x_1 ; x_2\right)$. And \citeauthor{feng2024analyzingunderstandinglimitationsdpo} have calculated the partial derivatives of Equation \ref{eq:15} with respect to $x_1$ and $x_2$:

\begin{equation}
\left\{\begin{array}{l}
\frac{\partial \mathcal{L}_{R}\left(x_1 ; x_2\right)}{\partial x_1}=-\frac{\beta x_2^\beta}{x_1\left(x_1^\beta+x_2^\beta\right)} \\
\frac{\partial \mathcal{L}_{R}\left(x_1 ; x_2\right)}{\partial x_2}=\frac{\beta x_2^{\beta-1}}{x_1^\beta+x_2^\beta} .
\end{array}\right.
\end{equation}

To evaluate the impact on generation probabilities for human-preferred and dispreferred responses, and to enable a systematic comparison with ASFT, we present a visualization of the optimization plane (loss landscape) and the gradient field in Figure \ref{fig:2}(\subref{fig:a}). Building upon the findings of \cite{feng2024analyzingunderstandinglimitationsdpo}, we distill several critical features of the gradient field, which we summarize into the following key aspects:

\begin{itemize}
    \item \textbf{Case 1} In the initial optimization stages, when $x_1$ is very small and $x_2$ is very large, LLMs tend to generate dispreferred responses. The BT gradient flow mainly increases $x_1$, with minor changes to $x_2$, as shown in the top left corner of Figure \ref{fig:2}(\subref{fig:a}).
    \item \textbf{Case 2} When $x_1$ and $x_2$ are both large, LLMs may generate both preferred and dispreferred responses. In this scenario, The BT gradient flow slightly increases $x_1$ and decreases $x_2$, leading to minimal overall changes and potentially trapping the optimization in saddle points. as illustrated in the top right corner of Figure \ref{fig:2}(\subref{fig:a}).
    \item \textbf{Case 3} When $x_2$ is very small, indicating limited LLM response variability, the BT gradient flow quickly decreases $x_2$ with minor adjustments to $x_1$, as depicted in the lower part of Figure \ref{fig:2}(\subref{fig:a}).
\end{itemize}

By analyzing the gradient vector field of BT model, we identified two limitations of BT model. 

On one hand, $\mathcal{L}_{R}\left(x_1 ; x_2\right)$ has a more significant impact on $x_2$ than on $x_1$, because the gradient with respect to $x_2$ is larger than that for $x_1$. As a result, $\mathcal{L}_{R}\left(x_1 ; x_2\right)$ decreases the probability of generating human-dispreferred data more rapidly than it increases the probability of generating preferred responses.

On the other hand, the initial state in the gradient vector field significantly influences the final optimization outcome. Specifically, the initial position of  LLMs with SFT might be located in the lower-left corner of Figure \ref{fig:2}(\subref{fig:a}), indicating low probabilities of generating both human-preferred and dispreferred responses, with the gradient direction not fully prioritizing the enhancement of preferred responses. Alternatively, if the initial position is in the upper-right corner of Figure \ref{fig:2}(\subref{fig:a}), the very small gradients present there can lead to slow convergence and difficulties in escaping local minima, potentially resulting in insufficient learning from human-preferred data. 

\subsection{Improvements in ASFT}

To analyze the optimization process of ASFT and compare it with the BT model, We apply the same transformation to the alignment loss $\mathcal{L}_{align}$ in ASFT, as shown in Equation \ref{eq:17}. Denote the transformed form as $\mathcal{L}_{align}\left(x_1 ; x_2\right)$. We then analyze its partial derivatives with respect to $x_1$ and $x_2$, respectively.

\begin{equation}\label{eq:17}
\begin{aligned}
\mathcal{L}_{align}= & -\log \sigma\left(f_\theta\left(x, y_\omega\right)\right)-\log \sigma\left(-f_\theta\left(x, y_l\right)\right) \\
= & -\log \sigma\left(\log \frac{x_1}{1-x_1}\right) \\
& -\log \sigma\left(-\log \frac{x_2}{1-x_2}\right)
\end{aligned}
\end{equation}

The partial derivatives of Equation \ref{eq:17} with respect to $x_1$ and $x_2$ are given by: 

\begin{equation}
\left\{\begin{array}{c}
\frac{\partial \mathcal{L}_{align}\left(x_1 ; x_2\right)}{\partial x_1}=-\frac{1}{x_1} \\
\frac{\partial \mathcal{L}_{align}\left(x_1 ; x_2\right)}{\partial x_2}=\frac{1}{1-x_2}
\end{array}\right.
\end{equation}

We compute the partial derivatives of the alignment loss function $\mathcal{L}_{align}$ with respect to $x_1$ and $x_2$, and construct the corresponding gradient fields in Figure \ref{fig:2}(\subref{fig:b}) to visually reveal the dynamic characteristics of ASFT. Our analysis shows that in all three scenarios previously discussed, the gradients in ASFT more effectively adjust $x_1$ and $x_2$, enabling LLMs to better learn to generate responses aligned with human preferences while avoiding those that are not.

\begin{itemize}
    \item For \textbf{Case 1}, the ASFT gradient flow simultaneously increases $x_1$ and decreases $x_2$, allowing the loss function to converge more quickly to the optimal solution.
    \item For \textbf{Case 2}, the ASFT gradient flow rapidly reduces $x_2$ with minimal change in $x_1$, helping LLMs quickly decrease the probability of generating responses that are misaligned with human preferences.
    \item For \textbf{Case 3}, the ASFT gradient flow rapidly increases $x_1$ with minimal change in $x_2$, aiding LLMs in quickly increasing the likelihood of generating responses that align with human preferences.
\end{itemize}

The alignment loss in ASFT has a balanced impact on both $x_1$ and $x_2$. As a result, ASFT is able to learn how to generate human-preferred responses while simultaneously learning to avoid generating responses that humans do not prefer.

For each pairwise preference data $\left(x, y_w, y_l\right) \in D$, the update rate of $x_1$ in $\mathcal{L}_{align}\left(x_1 ; x_2\right)$ with respect to $x_2$ is $\frac{1-x_2}{x_1}$. As $x_1$ tends to increase and $x_2$ tends to decrease during optimization, we have $\frac{1-x_2}{x_1} \rightarrow 1$. This indicates that $\mathcal{L}_{align}$ updates $x_1$ and $x_2$ in a balanced manner. As a result, ASFT effectively identifies and addresses the key factors in different scenarios during the optimization process. It prioritizes addressing the primary issues without showing a fixed bias toward either $x_1$ or $x_2$. This characteristic allows ASFT to avoid unnecessary actions, thereby enhancing optimization efficiency. In other words, ASFT guides LLMs to focus simultaneously on both improving responses that align with human preferences and mitigating those that do not.

Based on the characteristics of the gradient field of $\mathcal{L}_{align}$, we find that ASFT is not sensitive to the initial model. 

\begin{itemize}
    \item When the initial position of the LLMs is in the upper-right corner of Figure \ref{fig:2}(\subref{fig:b}), ASFT shifts its focus to reducing the probability of generating human-dispreferred responses. In this scenario, ASFT demonstrates its ability to quickly decrease this probability and prevent the LLMs from generating responses that are not aligned with human preferences.
    \item When the initial position of the LLMs is on the left side of \ref{fig:2}(\subref{fig:b}), ASFT shifts its emphasis to increasing the probability of generating human-preferred responses. In this case, ASFT is able to rapidly enhance this probability. Specifically, when the initial position of the LLMs is in the upper-left corner of Figure \ref{fig:2}(\subref{fig:b}), ASFT can simultaneously increase the probability of generating human-preferred responses while decreasing the probability of generating dispreferred responses.
\end{itemize}

In summary, by analyzing the gradient field of the alignment loss in ASFT, we conclude that ASFT does not suffer from the limitations of BT model, such as an excessive focus on dispreferred responses. This means that ASFT can increase the probability of generating human-preferred responses while decreasing the probability of generating responses that humans do not favor. Additionally, ASFT is not sensitive to the initial model, making it robust across various starting points.

\subsection{Analysis of Feasibility Without a Reference Model}

\begin{table*}[t]
\centering
\begin{tabular}{lccccc}
\toprule
& \textbf{Number of Examples} & \textbf{Baseline Model} & \textbf{Judge Model} & \textbf{Scoring Type} & \textbf{Metric} \\
\midrule
\textbf{Arena-Hard} & 500 & GPT-4-0314 & GPT-4 Turbo & Pairwise comparison & Win rate \\
\textbf{MT-Bench} & 80 & - & GPT-4/GPT-4 Turbo & Single-answer grading & Rating of 1-10 \\
\bottomrule
\end{tabular}
\caption{Details regarding the evaluation metrics for Arena-Hard and MT-Bench. The term ``baseline model'' denotes the model used for comparison.}
\label{tab1}
\end{table*}

Based on the previous analysis, we know that the loss function in BT models is sensitive to the reference model. The reference model, used as the initial model, carries inherent uncertainty, which can negatively impact optimization results, leading to suboptimal performance in fine-tuning large models for alignment. Specifically, after the SFT stage, the reference model is more likely to fall into \textbf{Case 2}, where the BT model's optimization process may struggle with getting stuck at a saddle point. In general, removing the reference model makes the initial position of LLMs unpredictable. When the initial model is located on the left side of Figure \ref{fig:2}(\subref{fig:a}), the focus of optimization should be on increasing the probability of generating human-preferred responses. Specifically, in \textbf{Case 3}, the BT model tends to focus more on reducing the probability of generating dispreferred responses, which can hinder optimization by providing insufficient gradients to maximize the probability of generating preferred responses. Therefore, despite the improved memory and computational efficiency from removing the reference model, the objective function under the BT model has inherent limitations, leading to instability in the results.

In ASFT, our objective function overcomes the limitations of the BT model, making it less sensitive to the initial model. Whether the initial model is on the left or right side of Figure \ref{fig:2}(\subref{fig:b}), ASFT can efficiently optimize towards the minimum. Thus, our approach enables effective optimization without the need for a reference model.

\section{Experiments}

\subsection{Experiment Setups}

\textbf{Model and training settings.} We applied our proposed algorithm to fine-tune the Llama3-8B-Instruct \cite{dubey2024llama3herdmodels} to validate its effectiveness. We observed that tuning hyperparameters is crucial for achieving optimal performance across all offline preference optimization algorithms \footnote{For ORPO, $\lambda$ is typically set to 0.1. For DPO, $\beta$ is commonly within the range of $[0.1,0.5]$.}. Specifically, for ASFT, setting the parameter $\beta$ within the range of 0.1 to 0.5 generally yields satisfactory results. In this experiment, $\beta$ is set to 0.1. The learning rate for the training process is set to 1e-05. 

\noindent\textbf{Reward dataset.} We employed the UltraFeedback \cite{cui2024ultrafeedbackboostinglanguagemodels} dataset, a GPT-4 annotated instruction tracking dataset containing approximately 64,000 instructions. This dataset is strategically designed to overcome the challenges associated with acquiring human feedback by automating the collection of high-quality AI feedback. Additionally, we utilized the HH-RLHF \cite{bai2022traininghelpfulharmlessassistant} dataset to evaluate the performance of our model on traditional text generation metrics. This approach enables a comprehensive assessment of the model's capabilities in generating coherent and contextually relevant responses under various conversational scenarios. 

\noindent\textbf{Baselines.} We compared ASFT with other offline preference optimization methods, including IPO \cite{azar2023generaltheoreticalparadigmunderstand}, DPO \cite{rafailov2024directpreferenceoptimizationlanguage}, and ORPO \cite{hong2024orpomonolithicpreferenceoptimization}. IPO is a theoretically grounded method that circumvents the assumptions made by DPO, specifically the notion that pairwise preferences can be substituted with point-based rewards. ORPO introduces an odd-ratio comparison term involving a reference-free model, enabling direct contrasts between winning and losing responses with the policy model, and is jointly trained with the SFT objective.

\noindent\textbf{Evaluation metrics.} We primarily employed two of the most popular open-ended instruction tracking benchmarks to evaluate our model: MT-Bench \cite{zheng2023judgingllmasajudgemtbenchchatbot} and Arena-Hard v0.1 \cite{li2024crowdsourceddatahighqualitybenchmarks}. These benchmarks assess the model's multifaceted conversational capabilities across a variety of queries and have been widely adopted by the community (for detailed information, see Table \ref{tab1}).

MT-Bench encompasses eight categories with a total of 80 questions. The recently released Arena-Hard is an enhanced version of MT-Bench, containing 500 well-defined technical problem-solving queries. We report scores according to the evaluation protocols specified for each benchmark. For Arena-Hard, we report the win rate (WR) relative to a baseline model. For MT-Bench, scores are derived using GPT-4 as the adjudicating model.

Additionally, we utilized the alpaca\_en \cite{alpaca} dataset as a test set to evaluate the performance of ASFT. This evaluation employed traditional text generation metrics, such as BLEU \cite{Papineni2002} and ROUGE\cite{Lin2004}, to compare the performance of ASFT against ORPO.

\subsection{Experimental Results}

\begin{table}[t]
\centering
\begin{tabular}{lccc}
\toprule
\multirow{3}{*}{\textbf{Method}} & \multicolumn{3}{c}{\textbf{Llama3-Instruct (8B)}} \\
\cmidrule(lr){2-4}
& \textbf{Arena-Hard} & \multicolumn{2}{c}{\textbf{MT-Bench}} \\
\cmidrule(lr){2-2} \cmidrule(lr){3-4}
& \textbf{WR (\%)} & \textbf{GPT-4 Turbo} & \textbf{GPT-4} \\
\midrule
SFT & 21.9 & 6.83 & 7.92 \\
DPO & 31.8 & \textbf{7.08} & 8.03 \\
IPO & 29.6 & 6.95 & 7.95 \\
ORPO & 28.4 & 7.02 & 8.01 \\
\midrule
\textbf{ASFT} & \textbf{32.5} & 7.05 & \textbf{8.05} \\
\bottomrule
\end{tabular}
\caption{Llama3-Instruct (8B) Results Comparison.We use off-the-shelf models as the SFT model}
\label{tab2}
\end{table}

\begin{table*}[t]
\centering
\begin{tabular}{ccccccc}
\toprule
\textbf{Dataset} & \textbf{Alignment Method} & \textbf{Model} & \textbf{BLEU-4} & \textbf{ROUGE-1} & \textbf{ROUGE-2} & \textbf{ROUGE-L} \\ 
\midrule
\multirow{3}{*}{\textbf{HH-RLHF}} & SFT & Llama3-8B-Instruction & 23.7 & 32.70 & 13.51 & 14.56 \\ 
& ORPO & Llama3-8B-Instruction & 35.91 & 41.42 & 20.50 & 27.84\\ 
& ASFT & Llama3-8B-Instruction & 37.70 & 43.71 & 22.42 & 30.09 \\ 
\midrule
\multirow{3}{*}{\textbf{UltraFeedback}} & SFT & Llama3-8B-Instruction & 25.62 & 36.28 & 15.66 & 16.73 \\ 
& ORPO & Llama3-8B-Instruction & 40.96 & 44.13 & 23.85 & 33.56 \\ 
& ASFT & Llama3-8B-Instruction & 42.43 & 46.60 & 25.38 & 34.98 \\ 
\bottomrule
\end{tabular}
\caption{BLEU and ROUGE Scores for Different Settings}
\label{tab3}
\end{table*}

\textbf{Difference in benchmark discrimination}

Although both MT-Bench and Arena-Hard are widely utilized, we observed that MT-Bench exhibits poor separability among different methods. The minor differences between methods in MT-Bench could be attributed to randomness, potentially due to its limited evaluation dataset size. In contrast, Arena-Hard demonstrates superior discriminative ability, providing clearer distinctions between the performances of various approaches.

\noindent\textbf{Efficiency of ASFT}

ASFT operates without the need for a reference model, making it lighter and simpler to implement than DPO and similar reference-dependent methods. This results in reduced computational expenses and enhances the efficiency of the training process. During model fine-tuning, ASFT reduces time by 13\% and decreases peak GPU memory usage by 16.7\% compared to reference-dependent methods (Figure \ref{fig:3}(\subref{fig:3a})).

\begin{figure}[h]
\centering
\subfloat[Likelihood margin.]{
  \includegraphics[width=0.2\textwidth]{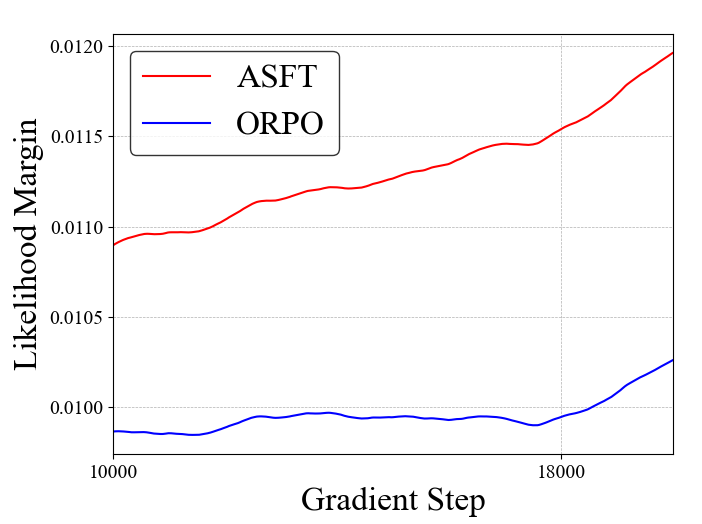}
  \label{fig:3a}
}
\subfloat[Runtime and GPU memory.]{
  \includegraphics[width=0.25\textwidth]{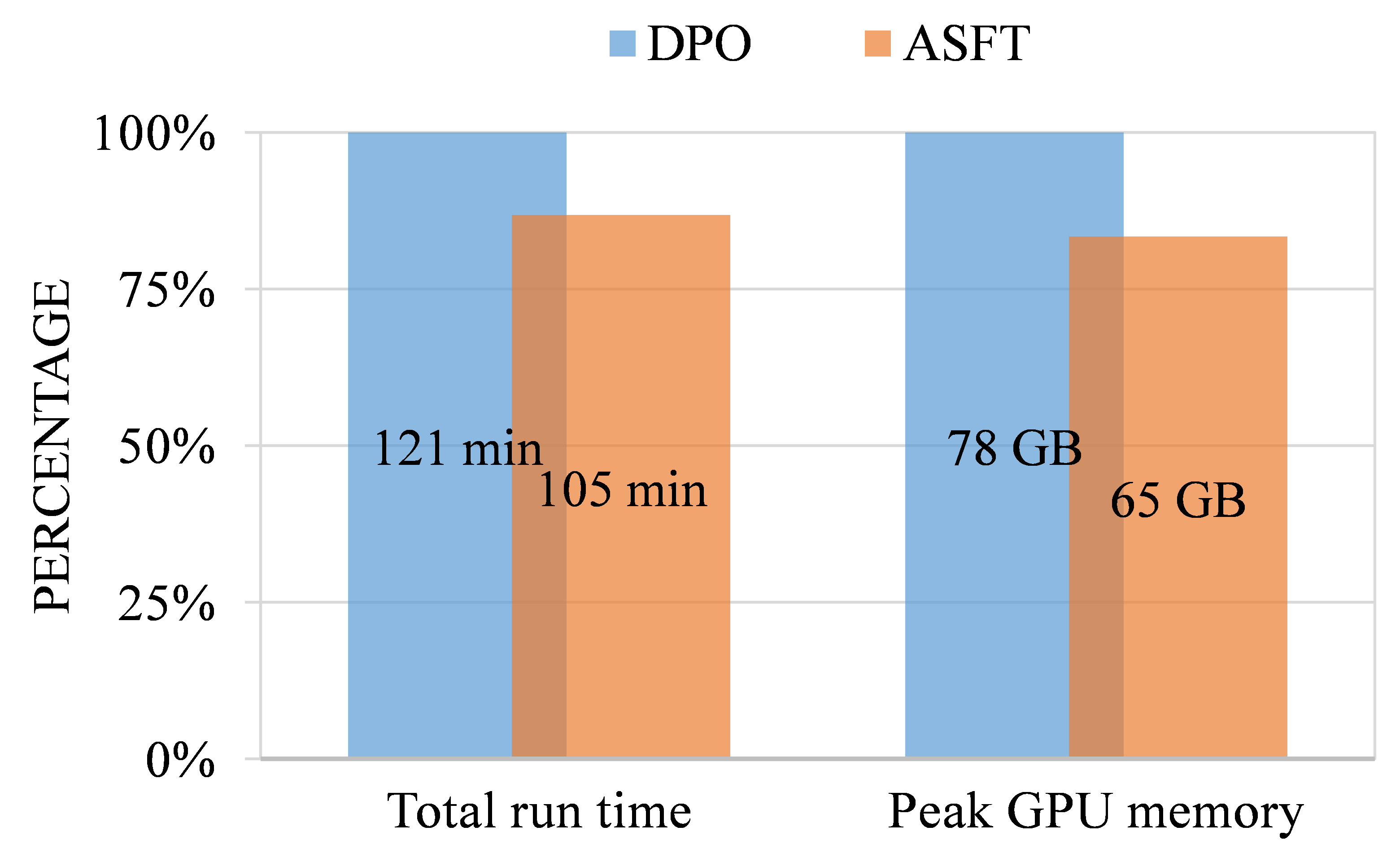}
  \label{fig:3b}
}%
\caption{Comparison measured on UltraFeedback. (a) Likelihood margin between ASFT and ORPO. (b) Total runtime and memory usage for ASFT and DPO}
\label{fig:3}
\end{figure}

\noindent\textbf{Effectiveness of ASFT}

In Table \ref{tab2} and \ref{tab3}, we present our main evaluation results. The experimental outcomes demonstrate that ASFT achieves comparable or improved performance over conventional alignment methods on the MT-Bench assessment. Notably, ASFT exhibits a substantial improvement on the Arena-Hard evaluation, compared to the baseline SFT model, there is a 48\% improvement in WR, underscoring its robustness. Furthermore, compared to ORPO and the base SFT model, ASFT's improvements in BLEU-4, ROUGE-1, ROUGE-2, and ROUGE-L metrics underscore a significant enhancement in the model's ability to follow instructions accurately. Additionally, The data likelihood margin of ASFT is significantly higher than that of ORPO (Figure \ref{fig:3}(\subref{fig:3a})). This demonstrates ASFT's effectiveness in refining the precision of language model responses across various evaluation frameworks.

\section{Limitations}
\textbf{Hyperparameter in experimental setting.} 

Due to the incorporation of the $\beta$ in the loss function, manual tuning is necessary. Currently, neither experimental data nor theoretical analysis has provided a definitive method for optimally setting the beta parameter across various scenarios. Future research could focus on developing methodologies to automate the determination of $beta$ value. This advancement would significantly improve the adaptability and efficacy of the approach.\\

\noindent\textbf{Harmlessness and honesty.}

Alignment methods are designed to adjust the output distributions of LLMs to improve their helpful, safety, and honesty. However, our current research primarily focuses on enhancing the utility of large models using the UltraFeedback dataset \cite{cui2024ultrafeedbackboostinglanguagemodels}, with less attention given to safety and honesty, which are essential for deploying LLMs across diverse domains. Future research should prioritize developing strategies to bolster both the safety and integrity of large models. From a robustness standpoint, even well-aligned models may become vulnerable when exposed to even minimal amounts of unsafe data. Additionally, fine-tuning aligned LLMs on benign datasets could inadvertently compromise their built-in safety protocols. Addressing these challenges is crucial for ensuring that enhancements in utility do not come at the expense of reduced safety or compromised ethical standards.

\noindent\textbf{Poor performance in mathematics}
Although the ASFT algorithm outperforms other methods in many aspects, we found that it exhibits a higher error rate when dealing with mathematical reasoning problems. This may be due, in part, to the Llama model's limited ability to solve mathematical problems in a zero-shot setting, and also because the preference optimization objective might not effectively increase the likelihood of preference sequences \cite{pal2024smaugfixingfailuremodes}, potentially hindering learning from mathematical preference pairs. Future work could explore strategies to address the suboptimal performance of large models on complex reasoning tasks.

\section{Conclusion}
In this work, we address the alignment of large models within a gradient optimization framework. By analyzing the gradients of alignment algorithms based on the BT model, we introduce a novel and practical approach: ASFT. This method aligns models using absolute likelihoods, overcoming the issue where algorithms loss function based on the BT model tends to decrease the probability of generating human-dispreferred data more rapidly than it increases the likelihood of producing preferred data. Additionally, we thoroughly discuss the theoretical justification for ASFT's operation without a reference model. The effectiveness of our approach is robustly validated through its application to the Llama3-8B model, demonstrating ASFT's significant potential in more effectively integrating language models with human feedback.
\bibliography{aaai25}

\appendix

\onecolumn

\section{Proof of Theorems}

\begin{theorem}

The partial derivatives of $\mathcal{L}_{R}\left(x_1, x_2\right)=-\log \left(\frac{x_1^\beta}{x_1^\beta+x_2^\beta}\right)$ with respect to $x_1$ and $x_2$ are given by:

\begin{equation}
\left\{\begin{array}{l}
\frac{\partial \mathcal{L}_{R}\left(x_1 ; x_2\right)}{\partial x_1}=-\frac{\beta x_2^\beta}{x_1\left(x_1^\beta+x_2^\beta\right)} \\
\frac{\partial \mathcal{L}_{R}\left(x_1 ; x_2\right)}{\partial x_2}=\frac{\beta x_2^{\beta-1}}{x_1^\beta+x_2^\beta} .
\end{array}\right.
\end{equation}

\end{theorem}

\begin{proof}

For $\frac{\partial \mathcal{L}_{R}\left(x_1 ; x_2\right)}{\partial x_1}$,

\begin{equation}
\begin{aligned}
\frac{\partial \mathcal{L}_{R}\left(x_1 ; x_2\right)}{\partial x_1} & =-\frac{x_1^\beta+x_2^\beta}{x_1^\beta}\left(\frac{\beta x_1^{\beta-1}}{x_1^\beta+x_2^\beta}+\frac{-\beta x_1^{2 \beta-1}}{\left(x_1^\beta+x_2^\beta\right)^2}\right) \\
& =-\frac{\beta x_1^{\beta-1}\left(x_1^\beta+x_2^\beta\right)-\beta x_1^{2 \beta-1}}{x_1^\beta\left(x_1^\beta+x_2^\beta\right)} \\
& =-\frac{\beta x_1^{\beta-1} x_2^\beta}{x_1^\beta\left(x_1^\beta+x_2^\beta\right)} \\
& =-\frac{\beta x_2^\beta}{x_1\left(x_1^\beta+x_2^\beta\right)}     
\end{aligned}
\end{equation}

For $\frac{\partial \mathcal{L}_{R}\left(x_1 ; x_2\right)}{\partial x_2}$,

\begin{equation}
\begin{aligned}
\frac{\partial \mathcal{L}_{R}\left(x_1 ; x_2\right)}{\partial x_2} & =\frac{1}{x_1^\beta+x_2^\beta} \beta x_2^{\beta-1} \\
& =\frac{\beta x_2^{\beta-1}}{x_1^\beta+x_2^\beta} 
\end{aligned}
\end{equation}

\end{proof}

\begin{theorem}

The partial derivatives of $\mathcal{L}_{align}\left(x_1, x_2\right)=-\log \sigma\left(\log \frac{x_1}{1-x_1}\right)-\log \sigma\left(-\log \frac{x_2}{1-x_2}\right)$ with respect to $x_1$ and $x_2$ are given by:

\begin{equation}
\left\{\begin{array}{l}
\frac{\partial \mathcal{L}_{align}\left(x_1 ; x_2\right)}{\partial x_1}=-\frac{1}{x_1} \\
\frac{\partial \mathcal{L}_{align}\left(x_1 ; x_2\right)}{\partial x_2}=\frac{1}{1-x_2}
\end{array}\right.
\end{equation}

\end{theorem}

\begin{proof}

We perform the following transformation on $\mathcal{L}_{align}\left(x_1, x_2\right)$ :

\begin{equation}
\begin{aligned}
\mathcal{L}_{align}\left(x_1, x_2\right) & =-\log \sigma\left(\log \frac{x_1}{1-x_1}\right)-\log \sigma\left(-\log \frac{x_2}{1-x_2}\right) \\
& =-\log \frac{x_1}{1-x_1+x_1}-\log \frac{1-x_2}{1-x_2+x_2}\\
& =-\log x_1-\log \left(1-x_2\right)
\end{aligned}
\end{equation}

The partial derivatives of $\mathcal{L}_{align}\left(x_1, x_2\right)$ with respect to $x_1$ and $x_2$ are given by:

\begin{equation}
\frac{\partial \mathcal{L}_{align}\left(x_1 ; x_2\right)}{\partial x_1}=-\frac{1}{x_1}
\end{equation}

and

\begin{equation}
\frac{\partial \mathcal{L}_{align}\left(x_1 ; x_2\right)}{\partial x_2}=\frac{1}{1-x_2}
\end{equation}

\end{proof}
\section{Pseudocode}

PyTorch code for the ASFT Loss is provided below:

\begin{verbatim}
import torch.nn.functional as F

def preference_loss(self, chosen_logps, rejected_logps, beta):
    """
    chosen_logps: Log-probabilities of the chosen options.
    
    rejected_logps: Log-probabilities of the rejected options.
    
    ASFT_loss: The total loss which combines the supervised fine-tuning loss 
    (sft_loss) and the scaled alignment loss (beta * log_f_theta). 

    log_f_theta: Computes a modified log odds ratio between the chosen and rejected 
    outputs. 
    
    sft_loss: The negative log probability of the chosen outputs, which represents 
    the supervised fine-tuning (SFT) loss.
    
    beta: A scaling factor applied to the log_f_theta, balancing the influence of the 
    alignment-focused loss against the supervised fine-tuning loss. 
    """
    
    log_f_theta = - F.logsigmoid(chosen_logps - torch.log1p(-torch.exp(chosen_logps))) 
    - F.logsigmoid( - rejected_logps + torch.log1p(-torch.exp(rejected_logps)))
    
    sft_loss = -chosen_logps
      
    ASFT_loss = sft_loss + beta * log_f_theta
    
    return ASFT_loss
\end{verbatim}

\newpage
\section{Experimental Details}
All models are fine-tuned from the publicly accessible meta-llama/llama3. Flash-Attention 2 (\cite{dao2023flashattention2fasterattentionbetter}) was applied for all the models during fine-tuning for computational efficiency. For optimizer, we use AdamW optimizer (\cite{loshchilov2019decoupledweightdecayregularization}) with a learning rating of 1e-05. And the linear warmup with cosine decay was applied for the learning rate. For input length, every instance was truncated and padded to 2,048 tokens for UltraFeedback. We train all models for 2 epoch. The batch size is 4. The default parameterization coefficient $\beta$ is 0.1.

\begin{figure}[h!]

\centering
\begin{tabular}{|p{2.5cm}|p{11.2cm}|}
\hline

\centering \textbf{Input} & 
Edit the following paragraph to correct any grammatical errors:She didn't remembre where is her purse, so I thinks its in the car but he's say it's on kitchen table but he are not sure, and then they asked me to looking for it, she's say, "Can you?", and I responds with, "Maybe, but ain't no sure," and he not heard me, and, "What?", he asks, "Did you found it?". \\ \hline

\centering \textbf{ASFT-Llama3 ultrafeedback} &
Here is the edited paragraph with corrected grammatical errors: 

She couldn't remember where her purse was, so I think it might be in the car, but he said it was on the kitchen table, but he wasn't sure. Then they asked me to look for it. She said, "Can you?" And I responded, "Maybe, but I'm not sure." He didn't hear me, and then he asked, "What?" "Did you find it?" he asked.
 \\ \hline

\centering \textbf{GPT-4 judge} &
The assistant did an excellent job in correcting the grammatical errors in the paragraph. They maintained the content and meaning of the original text while improving its structure, grammar, tense, and punctuation. The assistant's revisions produced a correctly formatted and grammatically accurate text that sounds very natural and is easy to understand. The assistant's answer is relevant, accurate, helpful, and shows attention to detail. 
 \\ \hline
\end{tabular}
\caption{A case study on mtbench, the Llama3-Instruct model trained by ASFT was used to give high quality responses to input, and GPT-4 was used as the evaluation model to give results.}
\end{figure}

\end{document}